\documentclass[]{article} 
\usepackage[utf8]{inputenc}
\usepackage{tikz-cd}
\usepackage{amsthm, amsmath, amssymb, amsfonts}
\usepackage{mathrsfs}
\usepackage{microtype}
\usepackage{bbm}
\usepackage{mathtools}
\usepackage{dsfont}
\usepackage{caption}
\usepackage{url}
\usepackage{authblk}
\usepackage{pifont}
\usepackage[linewidth=1pt]{mdframed}
\usepackage{framed}
\usepackage{lipsum}
\usepackage{caption}
\usepackage{apptools}
\usepackage{graphicx}
\usepackage{wrapfig}
\usepackage{caption}
\usepackage{subcaption}
\usepackage{boldline,multirow}

\usepackage[T1]{fontenc}    
\usepackage{url}            
\usepackage{booktabs}       
\usepackage{amsfonts}       
\usepackage{nicefrac}       
\usepackage{microtype}      

\usepackage[
backend=biber,
natbib=true,
url=false, 
doi=true,
eprint=false
]{biblatex}
\usepackage[nonatbib, preprint]{neurips_2019}
\usepackage{aliascnt}

\title{Copula \& Marginal Flows:\\ Disentangling the Marginal from its Joint}
\author[  \,\,,1,2]{Magnus Wiese\thanks{Corresponding author: \texttt{wiese@rhrk.uni-kl.de}}}
\author[2]{Robert Knobloch}
\author[1,2]{Ralf Korn}
\affil[1]{TU Kaiserslautern, Gottlieb-Daimler-Straße 48, 67663 Kaiserslautern, Germany
}
\affil[2]{Fraunhofer ITWM, Fraunhofer-Platz 1, 67663 Kaiserslautern, Germany
}
\date{\today}

\newcommand{\E}{\mathbb{E}}
\newcommand{\Prob}{\mathbb{P}}

\newcommand{\N}{\mathbb{N}}
\newcommand{\R}{\mathbb{R}}

\newcommand{\cond}{ \ | \ }
\newcommand{\normalbrack}[1]{\left( #1 \right)}
\newcommand{\sqbrack}[1]{\left[ #1 \right]}
\newcommand{\abs}[1]{\left| #1 \right|}

\newcommand{\norm}[1]{\left\|#1\right\|}

\newcommand{\Z}{\R^{d_0}}
\newcommand{\X}{\R^{d_1}}

\newcommand{\mathbfv}[1]{\mathbf}

\newtheorem{theorem}{Theorem}

\newaliascnt{corollary}{theorem}
\newtheorem{corollary}[corollary]{Corollary}

\aliascntresetthe{corollary}

\newaliascnt{definition}{theorem}
\theoremstyle{definition}
\newtheorem{definition}[definition]{Definition}

\aliascntresetthe{definition}
\newaliascnt{prop}{theorem}
\newtheorem{prop}[prop]{Proposition}

\aliascntresetthe{prop}
\newaliascnt{lemma}{theorem}
\newtheorem{lemma}[lemma]{Lemma}

\aliascntresetthe{lemma}
\newaliascnt{example}{theorem}
\theoremstyle{definition}
\newtheorem{example}[example]{Example}

\aliascntresetthe{example}
\newaliascnt{remark}{theorem}
\theoremstyle{remark}
\newtheorem{remark}[remark]{Remark}

\aliascntresetthe{remark}
\newaliascnt{notation}{theorem}
\theoremstyle{definition}

\aliascntresetthe{notation}
\newcommand{\Fbar}{\bar{F}}
\newcommand{\Fhat}{\hat{F}}
\newcommand{\thetab}{{\theta}}

\newcommand{\gen}{g_{\thetab}}
\newcommand{\geni}{g_{\thetab, i}}
\newcommand{\n}{N}
\newcommand{\z}{{Z}}
\newcommand{\x}{{X}}
\newcommand{\W}{\mathbf{W}}

\newcommand{\optgen}{\mathcal{G}^*}

\newcommand{\zrv}{$\z$}
\newcommand{\xrv}{$\x$}

\usepackage{apptools}

\begin{document}
	\maketitle

\begin{abstract}
	Deep generative networks such as GANs and normalizing flows flourish in the context of high-dimensional tasks such as image generation. However, so far an exact modeling or extrapolation of distributional properties such as the tail asymptotics generated by a generative network is not available. In this paper we address this issue for the first time in the deep learning literature by making two novel contributions. First, we derive upper bounds for the tails that can be expressed by a generative network and demonstrate $L^p$-space related properties. There we show specifically that in various situations an optimal generative network does not exist. Second, we introduce and propose \textit{copula and marginal generative flows} (\textit{CM flows}) which allow for an exact modeling of the tail and any prior assumption on the CDF up to an approximation of the uniform distribution. Our numerical results support the use of CM flows.
\end{abstract}

\section{Introduction}
Generative modeling is a major area in machine learning that studies the problem of estimating the distribution of a $\R^{d_1}$-valued random variable \xrv. One of the central areas of research in generative modeling is the model's (universal) applicability to different domains, i.e. whether the distribution of $X$ can be expressed by the generative model. Generative networks, such as generative adversarial networks (GANs) \cite{gans} and normalizing flows \cite{real_nvp,normalizing_flows}, comprise a relatively new class of unsupervised learning algorithms that are employed to learn the underlying distribution of \xrv \ by mapping an $\R^{d_0}$-valued independent and identically distributed (i.i.d.) random variable $Z$ through a parameterized generator $g_{{\theta}}: \Z \to\X$ to the support of the targeted distribution $\X$. Due to their astonishing results in numerous scientific fields - in particular image generation \cite{biggan, mescheder} - generative networks are generally considered to be able to model complex high-dimensional probability distributions that may lie on a manifold \cite{towards_principled}. This lets generative networks appear to be universally applicable generative models. We claim that this impression is misleading. 

In this paper we investigate to what extent this universal applicability is violated when dealing with distributions of random variables. Specifically we study the tail asymptotics generated by a generative network and the exact modeling of the tail. The right tail of \xrv \ is defined for $x \in \R $ and $i = 1,\dots, d_1$ as
\begin{equation*}
\Fbar_{X_i} (x) \coloneqq \Prob(X_i > x).
\end{equation*}
Similarly, the left tail is given by $\Fhat_{X_i}(x)\coloneqq \Prob(X_i < -x)$. The right tail function thus represents the probability that the random variable $X_i$ is greater than $x\in\R$.
\begin{figure}[htp]
	\centering
	\includegraphics[width=\textwidth]{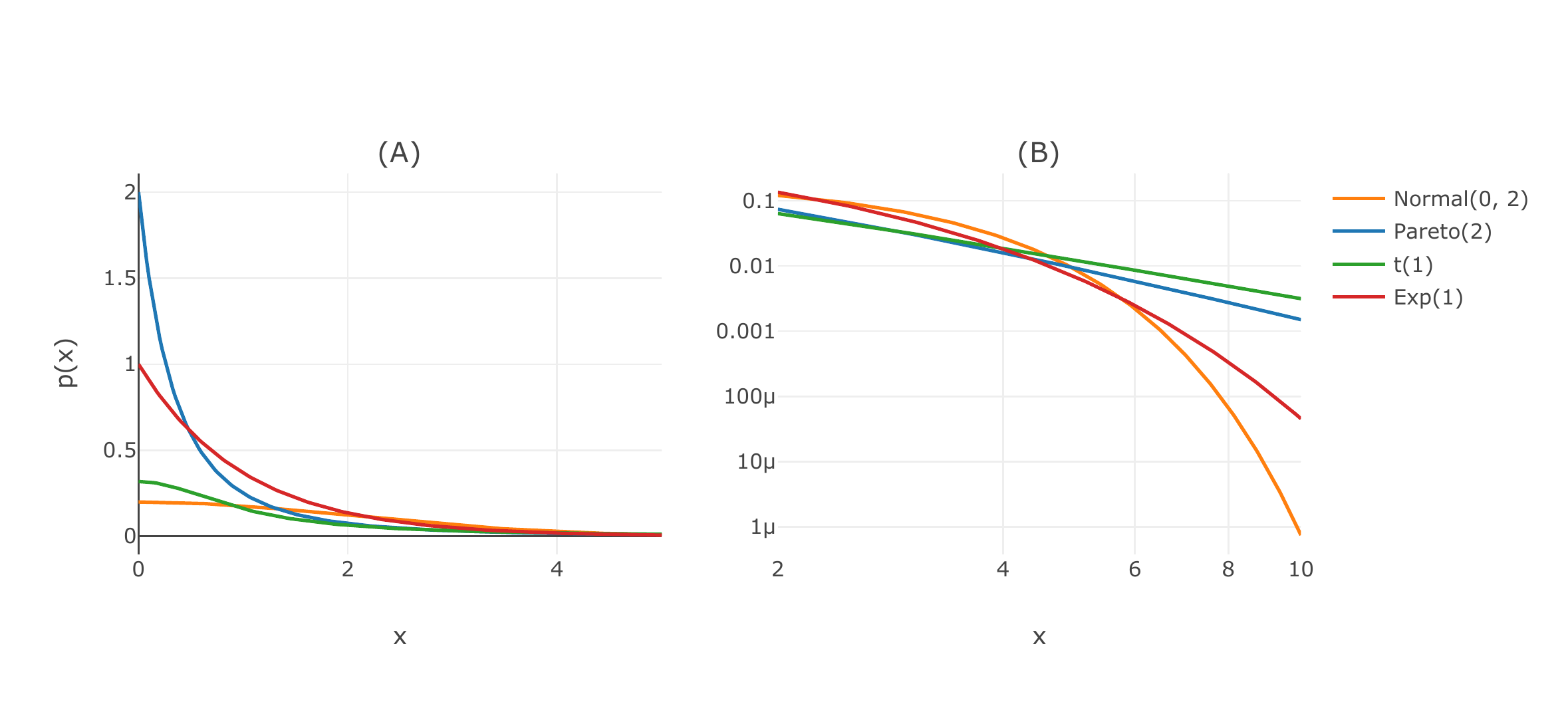}
	\caption{(A) and (B) depict the densities $f(x)$ of a Gaussian, Exponential, t- and Pareto distribution on the linear and log-scale respectively.}
	\label{fig:tail_comparison}
\end{figure}
In various applications \cite{floods, mc_methods_kkk, double_pareto, survival} it is essential to find a generative network that satisfies
\begin{equation}
F_{g_{\theta,i}(Z)}(x)  = F_{X_i}(x)
\label{eq:tail_goal_intuition}
\end{equation} 
for all $i = 1, \dots, d_1$ and $\abs{x} \gg 0$, since the asymptotic behavior determines the propensity to generate extremal values (see \autoref{fig:tail_comparison} for a comparison of the densities on the linear and logarithmic scale).\footnote{Although formally one should write $\gen \circ Z$ we abbreviate our notation as used above.} 

A typical situation that appears in practice is that only a sample of $X$ is available and for $i = 1, \dots, d_1 $ and $\abs{x} \gg 0$ the statistician would like the generative network to fulfill a \textit{tail belief} ${A_{X_i}:\bar\R\to[0,1]}$
\begin{equation}
F_{g_{\theta,i}(Z)}(x)  = A_{X_i}(x),
\label{eq:tail_goal_assumption}
\end{equation}
which may be derived by applying methods from extreme value theory \cite{de2007extreme}.
To the present moment no techniques exist in order to incorporate and model \eqref{eq:tail_goal_intuition} or \eqref{eq:tail_goal_assumption}; thus questioning the universal applicability of generative networks to domains where an extrapolation is necessary. 
Although one might resort to the universal approximation theorem for MLPs \cite{hornik} this is not applicable in practice as only a small proportion of the sample is found to be ``in the tail''. Thus the tail cannot be learned and \textit{has to be extrapolated}; whenever extrapolation is a central demand. 

\section{Main Results}
\label{sec:main_results}
In this paper we address the issue of modeling \textit{exact distributional properties} in the sense of \eqref{eq:tail_goal_intuition} and \eqref{eq:tail_goal_assumption} for the first-time in the context of generative networks. Our main results can be split into two parts. 

\subsection{Tail Asymptotics of Generative Networks}
In the first part we demonstrate that a generative network fulfilling \eqref{eq:tail_goal_intuition} does not necessarily exist. In particular we prove the following statement: 
\begin{theorem}
	\label{thm:tail_rate_bound}
	Let $g_\theta: \R^{d_0}\to\R^{d_1}$ be a parametrized generative network with Lipschitz constant $L(\theta)$ with respect to the $\norm{\cdot}_1$-norm\footnote{We define the $\norm{\cdot}_1$-norm for the vector space $\R^{d}, \ d\in \mathbb{N}$ as  $\norm{x}_1 \coloneqq \sum_{i}^d \abs{x_i}, \ x \in \R^d$.}. Furthermore, for $i = 1,\dots, d_1$ set $w_{\theta, i}(z) = d_0 \ L(\theta) \ z + \abs{g_{\theta, i}(0)}$.
	Then the generated tail asymptotics satisfy for $i = 1, \dots, d_1$
	\begin{equation*}
	\Fbar_{\abs{g_{\theta, i}(\z)}}(x) = \mathcal{O}\normalbrack{\Fbar_{w_{\theta, i}(\abs{Z_1})}(x)}
	\ \textrm{as} \ x \rightarrow \infty.
	\end{equation*}
\end{theorem}
The proof of \autoref{thm:tail_rate_bound} is provided in \autoref{sec:tail_bounds}. \autoref{thm:tail_rate_bound} gives rise to major implications such as the non-existence of an optimal generative network in various settings. Moreover, the result displays that the conception of choosing the noise prior $Z$ is negligible is false. 

In accordance to the derived tail bound the following $L^p$-space\footnote{An $\R^d$-valued random variable $X$ is an element of the space $L^p(\R^{d}, \norm{\cdot})$ if the expectation of $\norm{X^p}$ with respect to some norm $\norm{\cdot}$ on $\R^d$ is finite.} related property will be proven:
\begin{prop}
	\label{prop:lp_prop}
	Let $p \in \N$ and $g_\theta:\R^{d_0} \to \R^{d_1}$ be a parametrized generative network. If $Z$ is an element of $L^p(\R^{d_0}, \norm{\cdot})$, then $g_\theta(Z) \in L^p(\R^{d_1}, \norm{\cdot})$.
\end{prop}
The statement particularly demonstrates that a $p$-th unbounded moment cannot be generated when inferring a noise prior where the $p$-th moment is bounded. Finally, we conclude the section on tail bounds by observing that an exact modeling of the tails is not favored by utilizing generative networks.

\subsection{Copula and Marginal Flows}
\begin{figure}[t]
	\centering
\begin{subfigure}{.33\textwidth}
		\centering
		\includegraphics[width=\textwidth]{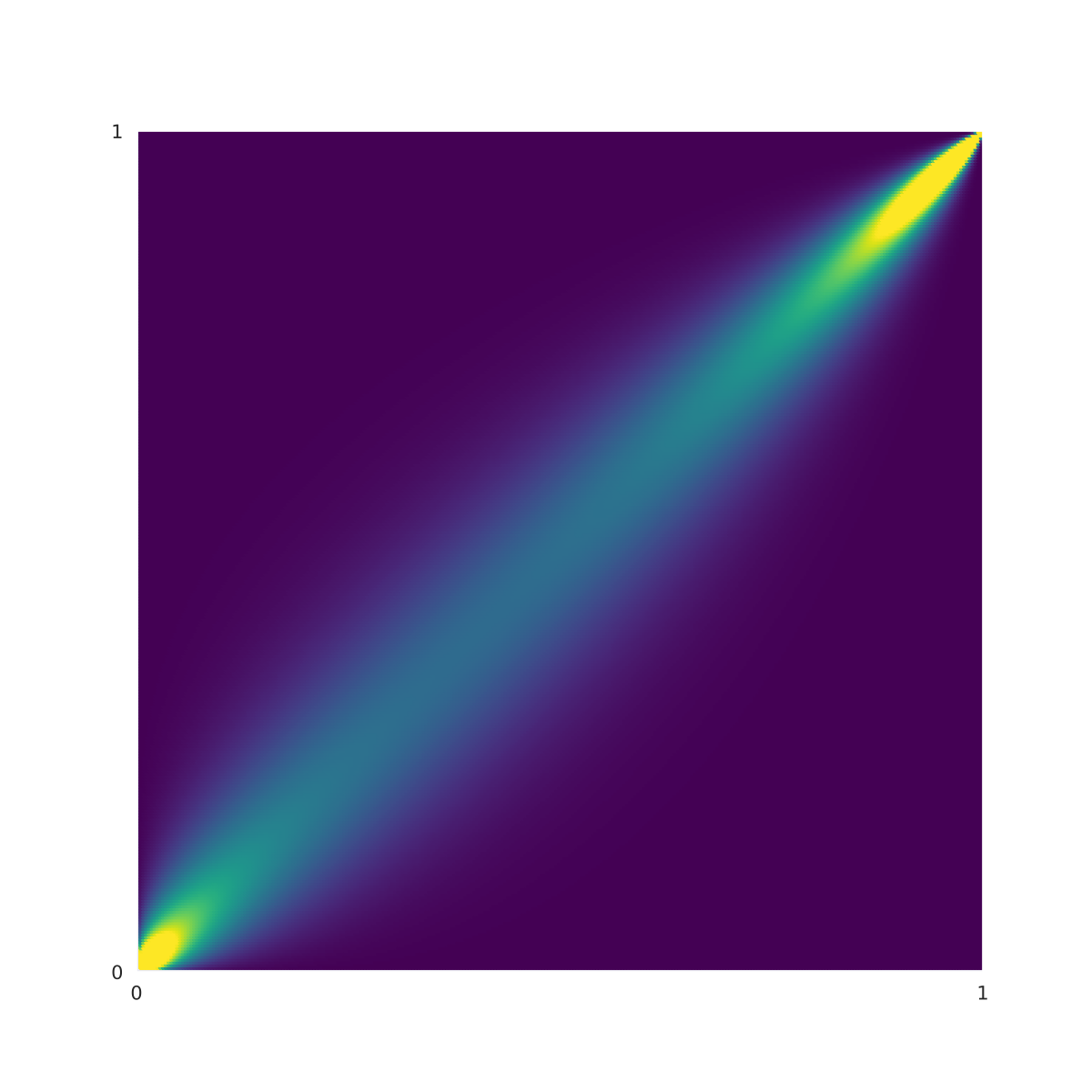}
		\caption{Gumbel Copula Density}
		\label{fig:gumbel_density}
\end{subfigure}%
\begin{subfigure}{.33\textwidth}
	\centering
	\includegraphics[width=\textwidth]{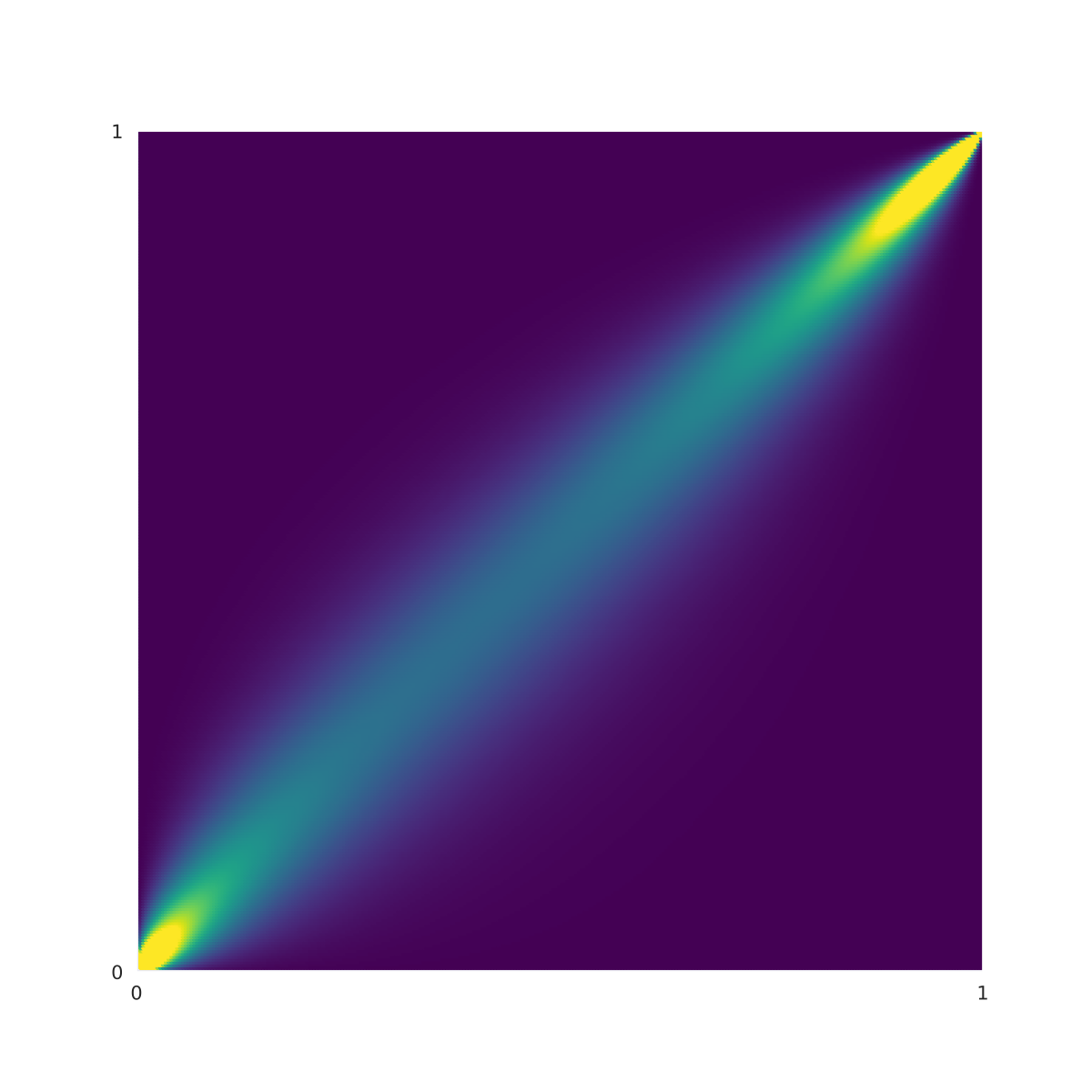}
	\caption{Copula Flow Approximation}
	\label{fig:copula_flow_gumbel}
\end{subfigure}%
\begin{subfigure}{.33\textwidth}
	\centering
	\includegraphics[width=\textwidth]{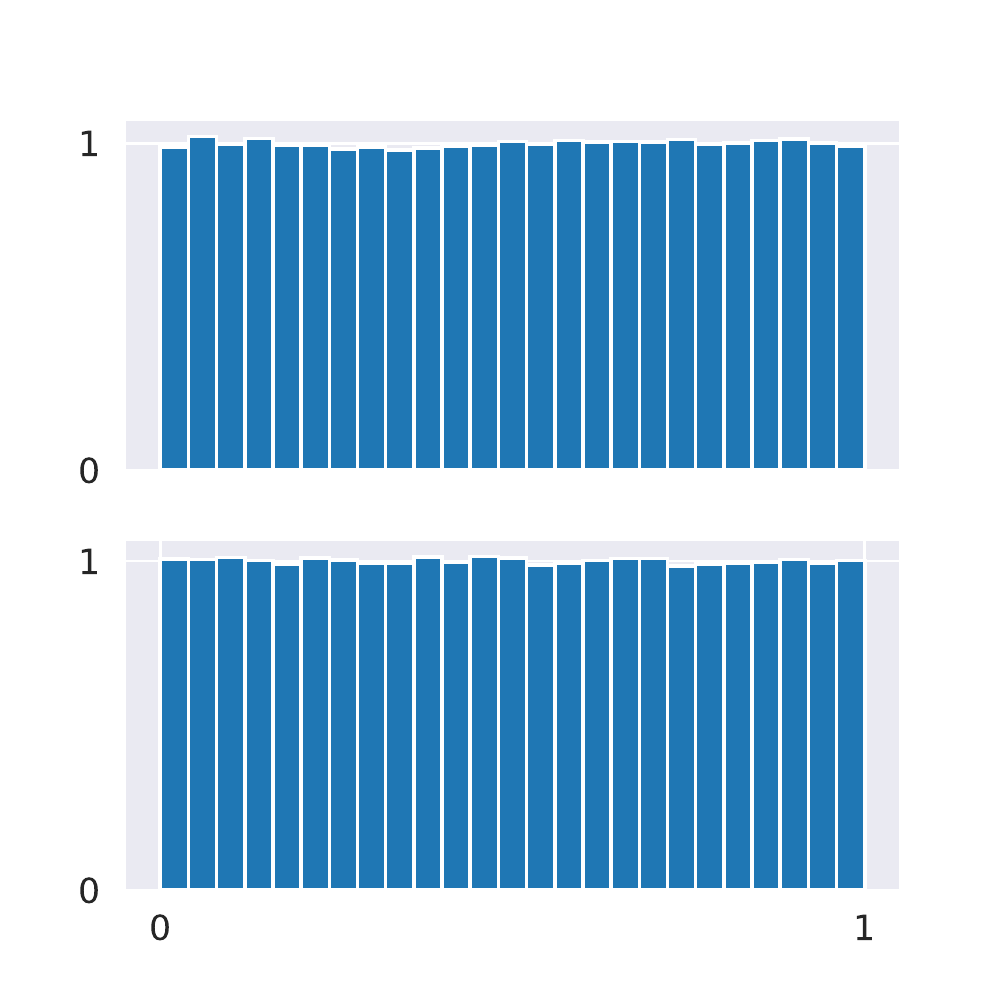}
	\caption{Copula Flow Marginals}
	\label{fig:gumbel_hist}
\end{subfigure}%
\caption{(a) illustrates the theoretical density of the gumbel copula, (b) the density obtained by applying a copula flow onto the Gumbel copula and (c) the histograms of each marginal obtained by sampling from the copula flow (5E+5 samples).}
\label{fig:gumbel}
\end{figure}

In the second part we introduce and propose the use of \textit{copula and marginal generative flows} (\textit{CM flows}) in order to model \textit{exact tail beliefs} whilst having a tractable log-likelihood. CM flows are a new model developed in this paper and are inspired by representing the joint distribution of $(X_1, X_2)$ as a copula plus its marginals, also known as a \textit{pair-copula construction} (\textit{PCC}) (cf. \cite{czado_pair_copula, nips_introduction_to_vine_copulas}):
\begin{align*}
p(x_1, x_2) &= p(x_2 \cond{} x_1)\  p(x_1)\\
&= c\normalbrack{F_{X_1}(x_1), F_{X_2}(x_2)} \ p(x_2) \ p(x_1),
\end{align*}
where $c:[0,1]^2 \to \R_{\geq 0}$ is the density of a copula. 

\begin{wrapfigure}{r}{0.5\textwidth}
	\[
	\begin{tikzcd}[column sep=huge,row sep=large]
	X  \arrow[rr,bend left,"g_{\theta, \eta}^{-1}"] \arrow[r,shift left=.75ex,"m_\theta^{-1}"] & 
	\hat C \arrow[r,shift left=.75ex,"h_\eta^{-1}"]& 
	\hat U
	\end{tikzcd}
	\]
	\[
	\begin{tikzcd}[column sep=huge,row sep=large]
	\tilde X  & 
	\tilde C \arrow[l,shift left=.75ex,"m_\theta"] & 
	\arrow[ll,bend left,"g_{\theta, \eta}"]
	\arrow[l,shift left=.75ex,"h_\eta"] 
	U
	\end{tikzcd}
	\]
	\caption{A commutative diagram of the proposed generative flow.}
	\label{fig:commutative_diagram}
\end{wrapfigure}
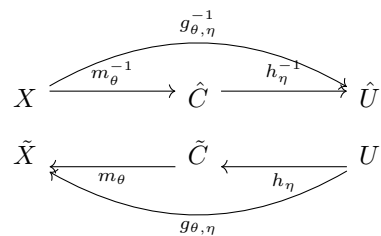
Following this decomposition CM flows are explicitly constructed by composing a copula flow $h_\eta: [0, 1]^2\to [0,1]^2 $ with a marginal flow $ m_\theta: [0,1]^2 \to \R^2$. The marginal flow approximates the inverse CDFs $F_{X_1}^{-1},\ F_{X_2}^{-1}$, whereas the copula flow approximates the generating function of $C \coloneqq (F_{X_1}(X_1), F_{X_2}(X_2))$. Thus, a CM flow is given by $g_{\theta, \eta}(u) = m_\theta \circ h_\eta(u) $ for $ u \in [0, 1]^2$ and the used transformations are depicted in \autoref{fig:commutative_diagram}. Although we restrict ourselves to introducing the bivariate CM flow, the proposed flow can be generalized by using Vine copulas \cite{bedford_vines, nips_introduction_to_vine_copulas}. 

The numerical results in \autoref{sec:numerical_results} highlight that bivariate copulas can be closely approximated by employing copula flows (see also \autoref{fig:gumbel} which depicts the results obtained by applying a copula flow onto the Gumbel copula.)

\subsection{Structure}
In \autoref{sec:tail_bounds} we derive upper bounds for the tails generated by a generative network. Afterward, we introduce in \autoref{sec:copula_and_marginal_flows} CM flows in order to model exact tails. Numerical results that support the application of copula flows will be presented in \autoref{sec:numerical_results}. Section \ref{sec:related_work} provides a literature overview and \autoref{sec:conclusion} concludes this paper.

\section{Related Work}
\label{sec:related_work}
PCCs were used in past works such as Copula Bayesian Networks (CBNs) \cite{cbns}. CBNs marry Bayesian networks with a copula-based re-parametrization that allow for a high-dimensional representation of the multivariate targeted density. However, since CBNs are defined through a space of copula and marginal densities they can only express joint-distributions defined in within their parametric space. Despite the amount of research that was committed in defining new families of copulas that are unsymmetrical and parameterizable it is still an active area of research. With CM flows we try to learn and represent the copula by optimizing a copula flow instead of dedicating a (restricted) parametric class as in CBNs. Therefore, our approach has more flexibility, however, comes at the cost of needing to approximate uniform marginals arbitrarily well. 

CM flows build up on the success of bijective neural networks by utilizing and modifying them. Therefore, our work relates in general to generative flows \cite{nice, real_nvp, NAF}. However, in this paper we discuss for the first time an exact modeling of distributional properties such as the tail of the targeted random variable. 

\section{Upper Tail Bounds for Generative Networks}
\label{sec:tail_bounds}
In this section we derive an upper bound for the tail induced as well as a $L^p$-space related property when feeding in the noise prior $Z$ into a generative network. Prior to proving these results we introduce our setup.

\subsection{Setup}
\label{sec:setup}
We begin by defining basic but important concepts which we believe to be a subset of the \textit{general assumptions} in deep learning literature. Roughly speaking, neural networks are constructed by composing affine transformations with activation functions (cf. \autoref{def:activation_function} in \autoref{appendix_definitions}). The main property that these networks have in general is that they are Lipschitz continuous. This is done for a good reason: gradients become bounded. We therefore define a network the following way: 
\begin{definition}[Network]
	\label{def:network}
	Let $d_0, d_1 \in \mathbb{N}$ and $\Theta$ be a real vector space. A function $f: \R^{d_0} \times \Theta \to \R^{d_1}$ that is Lipschitz continuous is called a \textit{network}. $\Theta$ is called the \textit{parameter space}. The space of networks mapping from $\R^{d_0}$ to $\R^{d_1}$ will be denoted by $\operatorname{DNN}\normalbrack{\R^{d_0}, \R^{d_1}}$.
\end{definition}
In the context of generative modeling we call a network a \textit{generative network} when it is defined as mapping from the \textit{latent space} $\R^{d_0}$ to the \textit{data/target space} $\R^{d_1}$ for $d_0, d_1 \in \mathbb{N},\ d_0 \leq d_1$. Furthermore, a $\Z$-valued random variable with i.i.d. components is called \textit{noise prior} and will be denoted by $Z$ throughout this section. The goal of generative modeling in the context of deep learning is to optimize the parameters $\theta \in \Theta$ of a generative network $g: \Z \times \Theta \to \X $ such that $g_{{\theta}}( Z)$, is equal in distribution to \xrv. This motivates our next definition.
\begin{definition}[Set of Optimal Generative Networks]
	Let $Z$ be a noise prior and $X$ an $\R^{d_1}$-valued random variable. We denote by 
	\begin{equation}
	\label{eq:domain}
	\optgen(\z,\x)\coloneqq \left\lbrace 
	g\in\operatorname{DNN}\normalbrack{\Z, \X}: \ \exists \theta \in \Theta^{(g)}: 
	\ \gen(Z) \stackrel{d}{=} X
	\right\rbrace,
	\end{equation}
	where $\stackrel{d}{=}$ represents equality in distribution, the \textit{set of optimal generators}.
\end{definition}

Generative networks by definition can represent any affine transformation. In \autoref{sec:upper_tail_bound} it will be useful to define the concept of \textit{affine lighter-tailedness} in order to compare the tails that can be generated by a network with other tails. 

\begin{definition}[Affinely Lighter-Tailed]
	Let $V$ and $W$ be two $\R$-valued random variables. 
	We call $V$ \textit{affinely lighter tailed than} $W$ iff for any affine function $r: \R \to \R$
	\[
	\Fbar_{\abs{r(V)}}(x) = o\normalbrack{\Fbar_{\abs{W}}(x)} \ \textrm{as} \ x\rightarrow\infty.
	\]
\end{definition}

\subsection{Results and Derivations}
\label{sec:upper_tail_bound}
In what follows we demonstrate that the tail generated by a network when inducing a noise prior $Z$ has order $\mathcal{O}(\Fbar_{w_\theta(\abs{Z_1})})$ where $w_\theta:\R \to \R$ is an affine transformation that depends on $\theta$. 
Prior to proving our main result \autoref{thm:tail_rate_bound} we show in \autoref{lemma:prob} that the tail $\Fbar_{Y}(x)$, where $Y = a\ \sum_{i=1}^{d_0} Z_i+b$ similarly has order $\mathcal{O}(\Fbar_{d_0 a Z_1}(x))$ as $x \rightarrow \infty$. A proof of the statement can be found in \autoref{appendix}.
\begin{lemma}
	\label{lemma:prob}
	Let $a, b \in \R$. Then for all $x \in \R$ we obtain
	\[
	\Prob\normalbrack{a \  \sum_{j=1}^{d_0} Z_j + {b} > x} \leq d_0 \ \Prob\normalbrack{d_0  \ a \ Z_1 +b > x}.
	\]
\end{lemma}
Next, we prove \autoref{thm:tail_rate_bound} by applying \autoref{lemma:prob} and utilizing Lipschitz continuity of networks. 
\begin{proof}[Proof of \autoref{thm:tail_rate_bound}]
	First observe that due to the Lipschitz continuity of a generative network the following property holds for all $z \in \Z$ and $i=1,\dots, d_1$:
	\begin{equation}
	\label{eq:lipschitz}
	\abs{g_{\theta, i}(z) - g_{\theta, i}(0)} \leq L(\theta)\norm{z} \ \Rightarrow \ \abs{g_{\theta, i}(z)} \leq L(\theta) \norm{z} + \abs{g_{\theta, i}(0)}.
	\end{equation}
	By applying \eqref{eq:lipschitz} and resorting to \autoref{lemma:prob} we obtain for all $x \in \R$
	\begin{align*}
		\Prob(\abs{\gen(\z)} > x) & \leq \Prob\normalbrack{ L(\theta) \ \sum_{i=1}^{d_0}\abs{Z_i} + \abs{g_\theta(0)} > x} \\
		& \leq 
		d_0 \cdot \Prob\normalbrack{w_\theta(\abs{\z_1}) > x}
	\end{align*}
	where $Z_1$ is the first component of the random variable $Z$. From this bound the order is a direct consequence and we can conclude the statement.  
\end{proof}

\autoref{thm:tail_rate_bound} has some immediate consequences. First, it shows that the tails of the distribution induced by the generator decay at least at the rate of an affine transformation of $\abs{Z_1}$. Therefore, if $X_i$ is not affinely lighter-tailed than $Z_1$ for some $i = 1,\dots, d_1$ the set of optimal generative networks is empty:
\begin{corollary}
	\label{cor:lighter_tail_prior}
	Assume that for some $i = 1,\dots, d_1$ the random variable $X_i $ is not affinely lighter-tailed than $Z_1$. Then $\optgen(Z,X)=\emptyset$.
\end{corollary}
The following two examples illustrate the effects of \autoref{cor:lighter_tail_prior} in two situations that are relevant both from a practical and theoretical perspective. 
\begin{example}
	Assume $V\sim\mathcal{N}(0, 1)$ is standard normally distributed and $Y$ an $\R$-valued Laplace distributed random variable. Then by \autoref{cor:lighter_tail_prior} the set of optimal generative networks $\optgen(V,Y)$ is empty.
\end{example}
\begin{example}
	Assume $U \sim \mathcal{U}([0,1])$ is uniformly distributed and $Y$ a random variable with support $\R$. Then $g_\theta(U)$ is bounded and again by \autoref{cor:lighter_tail_prior} we obtain $\optgen(U,Y) = \emptyset$.
\end{example}

Since the tail determines the probability mass allocated to extremal values it relates to the integrability of a random variable. We therefore arrive at \autoref{prop:lp_prop} which can be viewed as an $L^p$-space related characterization of the distribution induced by $g_\theta(Z)$. The result can be seen as another consequence of \autoref{thm:tail_rate_bound}, but will be proven for simplicity by applying the binomial theorem. 
\begin{proof}[Proof of \autoref{prop:lp_prop}]
	As in \eqref{eq:lipschitz} we obtain for a parametrized generative network $g_\theta: \R^{d_0} \to \R^{d_1}$, norm $\norm{\cdot}$ and for all $z \in \R^{d_0}$
	\begin{equation}
	\label{eq:lip_prop}
	\norm{g_\theta(z) - g_\theta(0)} \leq L(\theta) \norm{z} \Rightarrow \norm{g_\theta(z)} \leq L(\theta) \norm{z} + \norm{g_\theta(0)}
	\end{equation}
	due to $\norm{x}-\norm{y}\leq \norm{x-y}$ for $x,y\in\R^n$. Employing \eqref{eq:lip_prop} and applying the binomial theorem we can prove that $g_\theta(Z)$ is an element of the space $L^p(\R^{d_1}, \norm{\cdot})$
	\begin{align*}
	\E \sqbrack{\norm{g_\theta(Z)}^p} &\leq \E \sqbrack{\normalbrack{L(\theta) \norm{Z} + \norm{g_\theta( 0)}}^p}\\
	&= \sum_{k=0}^{p} {p \choose k} \  \E\sqbrack{L(\theta)^k \norm{Z}^{k}}  \ \norm{g_\theta(0)}^{p-k}\\
	&<\infty,
	\end{align*}
	where we used that $Z$ is an element of the space $L^p(\Z, \norm{\cdot})$. This proves the statement. 
\end{proof}

\subsection{The Inability of Estimating and Adjusting the Tailedness}
In order to estimate and consequently adjust the tail by exchanging the noise prior $Z$ we would need besides the Lipschitz constant $L(\theta)$ for all $i = 1, \dots, d_1$ a ``lower'' Lipschitz constant $K_i(\theta)$ which is defined for $z \in \Z$  as
\begin{equation}
\label{eq:up_low_lip_bounds}
K_i(\theta) \norm{z}_1 \leq \geni(z) - \geni(0) \leq L(\theta) \norm{z}_1.
\end{equation}
With \eqref{eq:up_low_lip_bounds} and \autoref{thm:tail_rate_bound} a lower and upper bound of $\Fbar_{\geni(Z)}$ could be obtained. However, since in general $K_i(\theta)$ is not available we arrive at the result that the induced tail remains unknown and thus, the statistician unpleased.\footnote{We note that in simplified network constructions such as in a ReLU network $f:\Z \to \R$ the exact tail can be obtained for one-dimensional targeted random variables $X$ by using \cite[Lemma 1]{sizenoise} which shows that the domain of $f$ can be divided into a finite number of convex pieces on which $f$ is affine.}

\section{Copula and Marginal Flows: Model Definition}
\label{sec:copula_and_marginal_flows}
The previous section on tail bounds demonstrates that generative networks do not favor controlling the generated tail behavior. We now show how a tail belief can be incorporated by using a bivariate CM flow $g: [0,1]^2 \times \Theta \times H \to \R^2$ that is defined for $u \in [0,1]^2$ as the composition of a parametrized marginal and a copula flow
\[
g_{\theta, \eta}(u) = m_\theta \circ h_\eta(u).
\]
The bivariate marginal flow is represented by two DDSFs (cf. \cite{NAF} or \autoref{def:ddsf}), whereas the bivariate copula flow by a 2-dimensional Real NVP \cite{real_nvp}. Although only the bivariate case is introduced we remark that CM flows can be generalized to higher dimensions by following a Vine copula and leave this as future work. Throughout this section we assume that $X=(X_1, X_2)$ is $\R^2$-valued and $F_{X_1}, F_{X_2}$ are invertible. 

\subsection{Marginal Flows: Exact Modeling of the Tail}
In what follows we construct univariate marginal flows for the $\R$-valued random variable $X_1$ and then define the vector-valued extension for $X$. Beforehand let us specify the concept of tail beliefs. 
\begin{definition}[Tail Belief]
Let $\alpha,\beta \in \bar\R, \alpha < \beta$ and set $B_X \coloneqq (-\infty, \alpha] \cup [\beta, \infty)$. Furthermore, let $X$ be an $\R$-valued random variable and $A_{X}: \bar\R\to[0,1]$ a known CDF. We call the tuple $(A_X, B_X)$ \textit{tail belief} when we assume that 
\[
\forall x \in B_X: \quad F_{X}(x) = A_{X}(x).
\]
\end{definition}
Thus, when incorporating a tail belief into a generative network we are interested in finding a mapping $m_{\theta}: [0,1] \to \R$ for $U_1\sim \mathcal{U}([0,1])$ that satisfies 
\begin{equation}
\forall x \in B_{X_1}: \quad \Prob(m_{\theta}(U_1) \leq x)  = A_{X_1}(x) \label{eq:tail_goal}.
\end{equation}
In order to satisfy \eqref{eq:tail_goal} we propose the following construction.
\begin{definition}[Univariate Marginal Flow]
	\label{def:univariate_marginal_flow}
	Let $a = A_{X_1}(\alpha)$ and $b = A_{X_1}(\beta)$. Furthermore, let $f:\R \times \Theta \to \R$ be a DDSF (cf. \autoref{def:ddsf}) and $\tilde f: [a,b] \times \Theta \to [\alpha, \beta]$ a scaled version of $f_{}$ which is defined as
	\begin{align*}
	\tilde f_{}(u,\theta) = (\beta-\alpha) \cdot \dfrac{f_{\theta}(u)-f_{\theta}(a)}{f_{\theta}(b)-f_{\theta}(a)} + \alpha.
	\end{align*}
	We call a function $m: [0,1] \times \Theta \to \R$ defined as
	\begin{equation}
	m(u, \theta) =
	\begin{cases}
	A_{X_1}^{-1}(u) \quad &u \in [0, a] \cup [b, 1]\\
	\tilde f(u,\theta) \quad &u\in (a,b)
	\end{cases}
	\end{equation}
	a \textit{univariate marginal flow}.
\end{definition}
\paragraph{Properties}
By construction a parametrized univariate marginal flow $m_\theta$ defines a bijection and satisfies the tail objective \eqref{eq:tail_goal}. Furthermore, the construction can be generalized in order to incorporate any prior knowledge of $F_{X_1}$ on a union of compact intervals. While $m_\theta$ defines an optimal map on $[0, a] \cup [b, 1]$, the flow approximates the inverse CDF $F_{X_1}^{-1}$ on $(a,b)$ and therefore only needs to be trained on the interval $(a,b)$. 

Due to the invertibility of a univariate marginal flow $m_\theta$ the density of a sample $x \in (\alpha, \beta), \ m_\theta(u) = x$ can be evaluated by resorting to the \textit{change of variable formula} \cite{bauer2002wahrscheinlichkeitstheorie, real_nvp}. 
Thus the parameter $\theta$ can be optimized by minimizing the negative log-likelihood (NLL) of $p(x)$ for $x \in (\alpha, \beta)$ while discarding any samples $x \not\in (\alpha, \beta) $.
\paragraph{Bivariate Marginal Flows}
Generalizing univariate marginal flows to bivariate (or multivariate) marginal flows is simple. For this we assume that for $i = 1, 2$ we have a tail belief $(A_{X_i}, B_{X_i})$. Then for each $i=1, 2$ we can define an univariate marginal flow $m^{(i)}: [0,1]\times \Theta^{(i)} \to [0,1] $ and construct the multivariate {marginal flow} which is defined for $u \in [0,1]^d$ as 
\[
m_\theta(u) = \sqbrack{m_{\theta_1}^{(1)}(u_1), m_{\theta_2}^{(2)}(u_2)}^T.
\]

\subsection{Copula Flows: Modeling the Joint Distribution}
Bivariate marginal flows were constructed to approximate the inverse CDFs $F_{X_1}^{-1}, F_{X_2}^{-1}$. We now define bivariate copula flows in order to approximate the generating function of 
\[
(C_1, C_2) \coloneqq (F_{X_1}(X_1), F_{X_2}(X_2)),
\]
whilst having a tractable log-likelihood. 
\begin{definition}[Bivariate Copula Flow]
	Let $\tilde h: \R^{2} \times H \to \R^{2}$ be a Real NVP (cf. \cite{real_nvp}) and $\Psi: \R \to [0,1]$ an invertible CDF. A function defined as
	\begin{align*}
	h: [0,1]^{2} \times H &\to [0,1]^{2}\\
	(u, \eta) &\mapsto \Psi \circ \tilde h_\eta \circ \Psi^{-1}(u)
	\end{align*}
	where $\Psi^{-1}$ and $\Psi$ are applied component-wise, is called \textit{copula flow}.
\end{definition}
\paragraph{Properties} Bivariate copula flows are bijective, since they are compositions of bijective functions. Furthermore, by applying the CDF $\Psi$ after the generative flow, the output becomes $[0,1]^{2}$-valued. In our implementation we use $\Psi = \sigma$, where $\sigma$ is the sigmoid activation. The copula flow's objective is to optimize the parameters $\eta \in H$ such that for $U \sim \mathcal{U}([0,1]^2)$ the random variable $(\tilde C_1, \tilde C_2) = (h_{\eta, 1}(U), h_{\eta, 2}(U))$ closely approximates $(C_1, C_2)$. 

As a special case bivariate copula flows can be parametrized such that only one variable is transformed whereas the other stays identical
\begin{equation}
\label{eq:constrained_cf}
(\tilde C_1,  \tilde C_2)  = (h_{\eta, 1}(U),  U_2),
\end{equation}
ensuring that the marginal distribution of the second component is uniform and thus the modeling of exact tails. We refer to \eqref{eq:constrained_cf} as a \textit{constrained bivariate copula flow}. 

\section{Numerical Results}
\label{sec:numerical_results}
Due to the positive results of DDSFs \cite{NAF} and thus the effectiveness of marginal flows we restrict ourselves to the evaluation of copula flows. Specifically, we evaluate the generative capabilities of copula flows on three different tasks, generating the Clayton, Frank and Gumbel copula. Before we report our results we introduce the following metrics and divergences used to compare the distributions.
\subsection{Metrics and Divergences}
The first performance measure we track is the Jensen-Shannon divergence (JSD) of the targeted copula $C$ and the approximation $\tilde C$, which we denote by $\operatorname{JSD}(C \ \| \ \tilde C)$. Furthermore, to assess whether the marginal distributions generated by the copula flow are uniform we approximate via Monte Carlo for $i=1,2$ and $A_k = [{(k-1)}/{n}, \ k/n), \ k =1, \dots, n$ the metric
\[
\operatorname{T}(i, n) \coloneqq \dfrac{1}{n}\sum_{k = 1, \dots, n}\abs{\log{\Prob({\tilde C_i \in A_k})} + \log n}
\]
where $\tilde C_i = h_{\eta, i}(U)$. Last, we also compute the maximum of each density estimator
\[
\operatorname{M}(i, n) \coloneqq \max_{k = 1, \dots, n}\abs{\log\Prob({\tilde C_i \in A_k}) + \log n }.
\]
\subsection{Benchmarks}
The results obtained by training a copula flow on the different theoretical benchmarks can be viewed in \autoref{tab:benchmark_copula}. In order to optimize the copula flow we used a batch size of 3E+03. The training of the copula flow was stopped after a breaking criterion was obtained, which we defined by setting thresholds for each of the metrics in \autoref{tab:benchmark_copula}. The metrics $\operatorname{T}$ and $\operatorname{M}$ were evaluated by using a batch size of 5E+05, the NLL with a batch size equal to the training batch size, and the Jensen-Shannon divergence was obtained by evaluating the theoretical and approximated density on a (equidistant) mesh-grid of size $300 \times 300$.\footnote{Due to numerical instabilities of the theoretical copula densities evaluations with nans were discarded.}

\autoref{fig:clayton} and \autoref{fig:frank} illustrate the theoretical and approximated density on the Clayton and Frank copula benchmark. \autoref{fig:jsd_pointwise} in \autoref{appendix} depicts the Clayton, Gumbel and Frank copula when evaluating the JSD pointwise on a mesh-grid of size $300 \times 300$. Lighter colors depict areas that were approximated not as well.

\begin{figure}
	\begin{minipage}{.5\textwidth}
		\centering
		\begin{subfigure}{.4\textwidth}
			\centering
			\includegraphics[width=\textwidth]{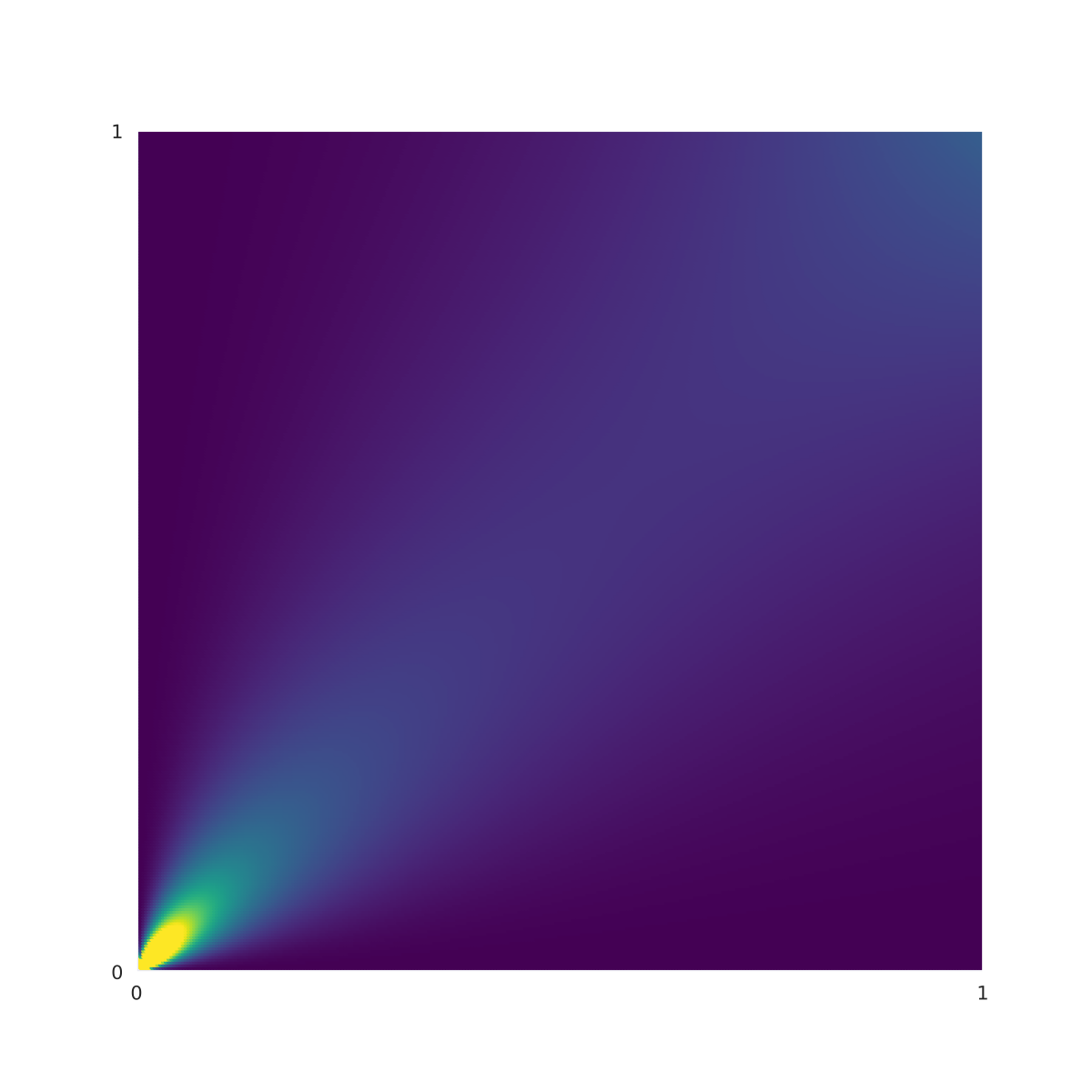}
			\caption{Theoretical}
			\label{fig:clayton_density}
		\end{subfigure}
		\begin{subfigure}{.4\textwidth}
			\centering
			\includegraphics[width=\textwidth]{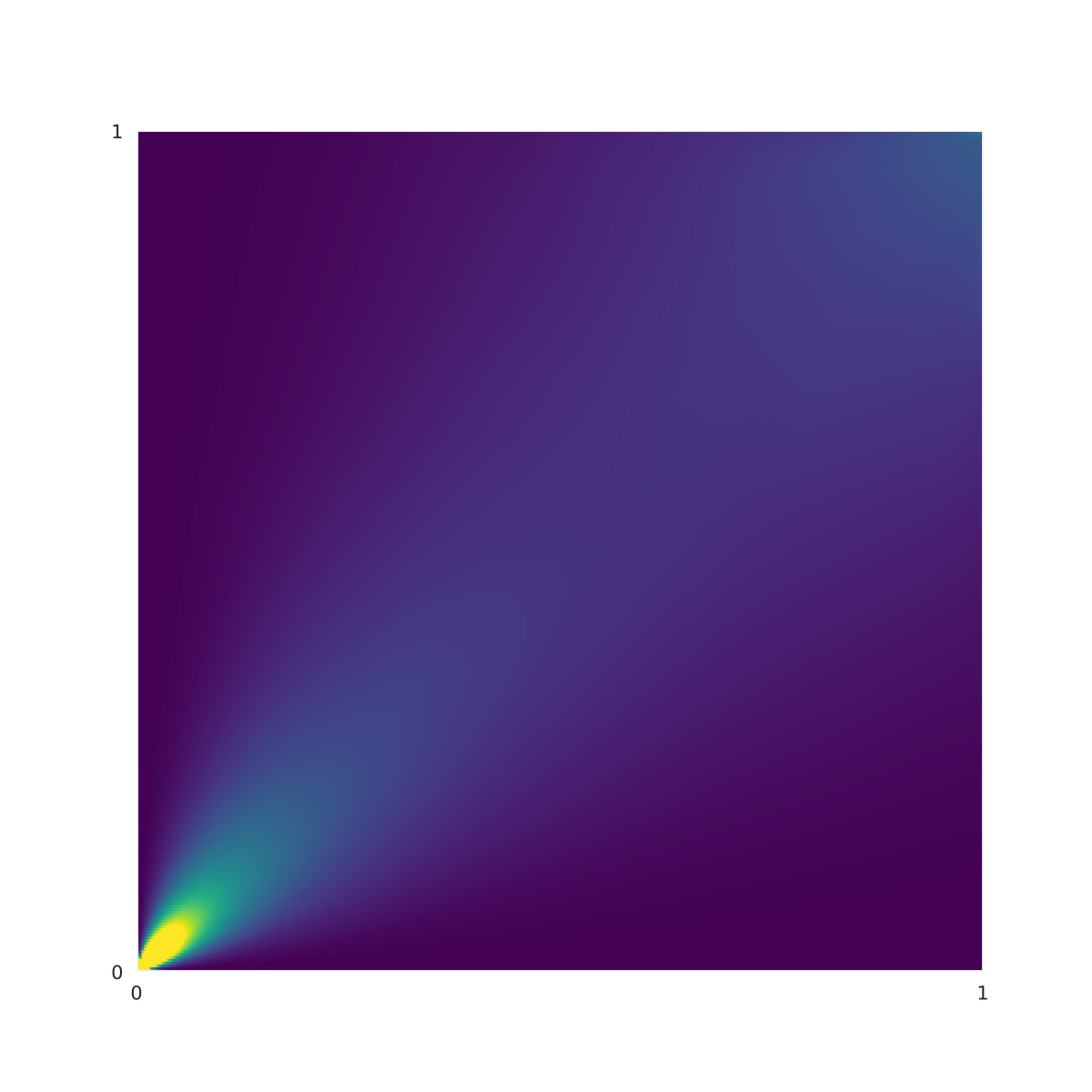}
			\caption{Approximation}
			\label{fig:copula_flow_clayton}
		\end{subfigure}
	\caption{Clayton Copula}
	\label{fig:clayton}
	\end{minipage}
\begin{minipage}{.5\textwidth}
	\centering
	\begin{subfigure}{.4\textwidth}
		\centering
		\includegraphics[width=\textwidth]{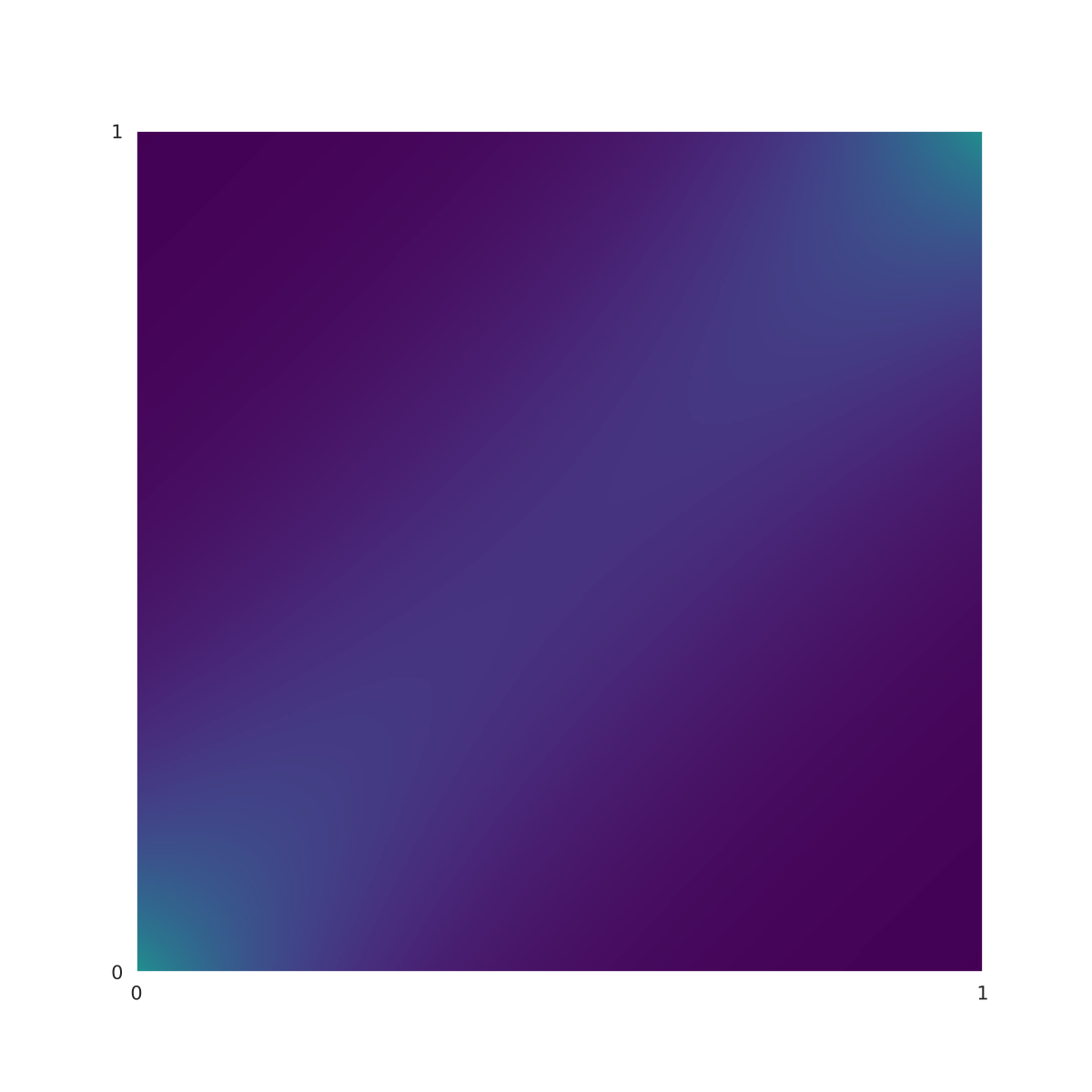}
		\caption{Theoretical}
		\label{fig:frank_density}
	\end{subfigure}
	\begin{subfigure}{.4\textwidth}
		\centering
		\includegraphics[width=\textwidth]{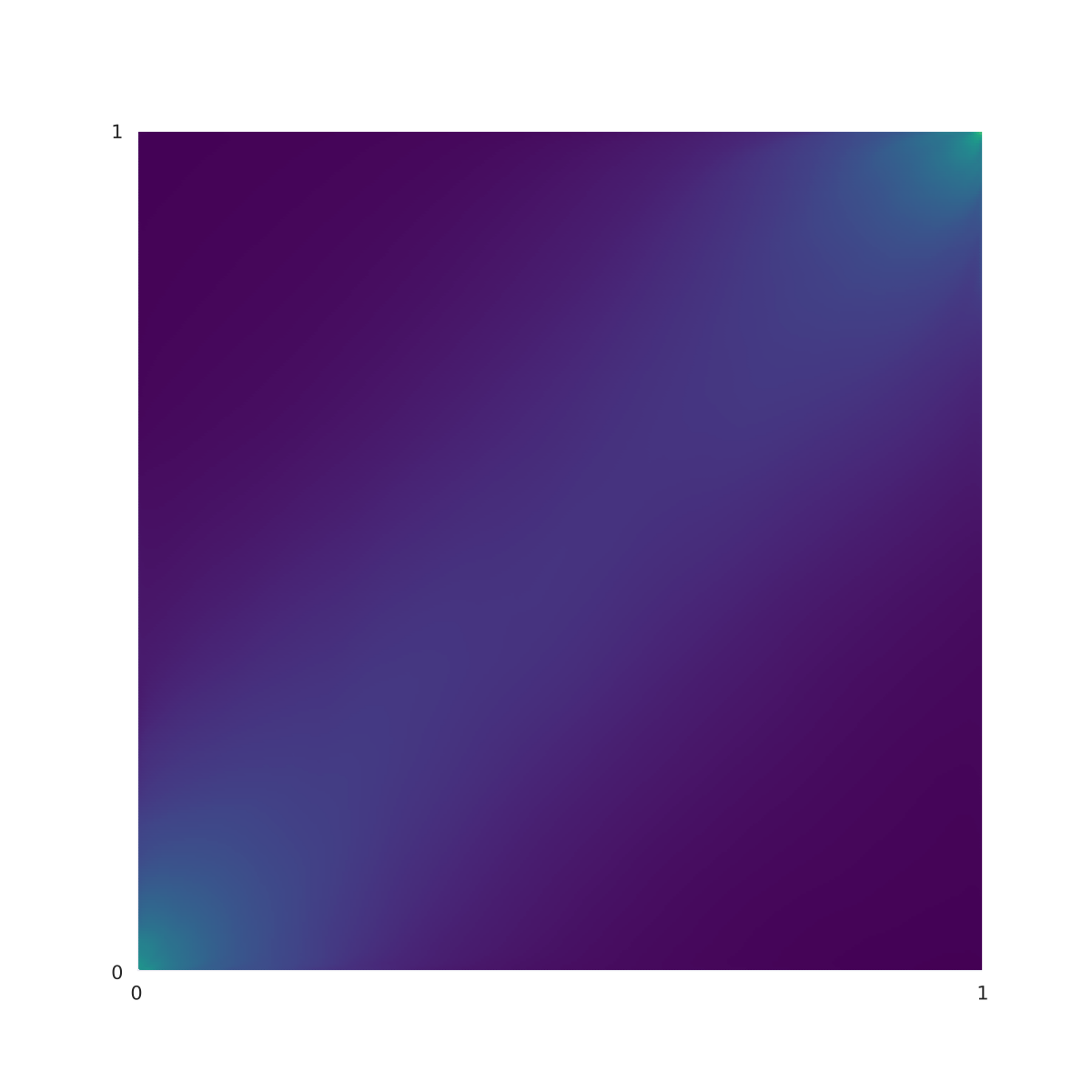}
		\caption{Approximation}
		\label{fig:copula_flow_frank}
	\end{subfigure}
	\caption{Frank Copula}
	\label{fig:frank}
\end{minipage}
\end{figure}

\begin{center}
	\begin{tabular}{ c c c c c c c}
		\hline
		Copula / Metrics & $\operatorname{JSD}( C \ \| \ \tilde C)$&  $\operatorname{T}(1, 25)$ & $\operatorname{T}(2, 25)$  & $\operatorname{M}(1, 25)$ & $\operatorname{M}(2, 25)$ & $\operatorname{NLL}$
		\\\clineB{1-7}{2.5}
		Clayton(2) &2.40E-04 & 1.17E-04 &1.27E-04 &4.04E-02& 5.06E-02& -4.41E-01
		\\\hline
		Frank(5) &6.89E-04&1.64E-04&1.71E-04 &4.98E-02&4.68E-02&-2.56E-01
		\\\hline
		Gumbel(5)* &5.96E-03 &1.18E-04&1.01E-04&3.44E-02&3.63E-02&-1.22E+00
		\\
		\clineB{1-7}{2.5}
	\end{tabular}
\captionof{table}{Performance figures obtained on four different copula benchmarks. The number in brackets after each copula type specifies the parameter being used. The star indicates that it was possible to apply a restricted bivariate copula flow.}
\label{tab:benchmark_copula}
\end{center}

\section{Conclusion}
\label{sec:conclusion}
In this paper we demonstrated that generative networks do not favor an exact modeling nor an estimation of the tail asymptotics. Since in various applications an exact modeling of the tail is of major importance we introduced and proposed CM flows. CM flows were explicitly constructed by using a marginal and copula flow which build up on the success of DDSFs and Real NVPs. The numerical results empirically demonstrated that bivariate copulas can be closely approximated by a copula flow and thus support the use of CM flows.

For CM flows to flourish we leave it as future work to correct the marginal distributions induced by a copula flow to be uniform. Once this is achieved, an exact modeling of the marginal distribution will be possible and tails can be modeled in unprecedented ways with deep generative flows. 

\printbibliography
\newpage
\appendix 
\section{Proofs}
\label{appendix}

\begin{proof}[Proof of \autoref{lemma:prob}]
	\begin{align*}
	\Prob\normalbrack{a\ \sum_{j=1}^{d_0} Z_j + b > x} &=
	1-\Prob\normalbrack{a \sum_{j=1}^{d_0} Z_j \leq {x-b}{}} \\
	& \leq 1 - \Prob\normalbrack{\bigcap_{j=1}^{d_0} \left\lbrace a \ Z_j \leq \dfrac{x - b}{d_0} \right\rbrace }\\
	&=\Prob\normalbrack{\bigcup_{j=1}^{d_0} \left\lbrace a \ Z_j > \dfrac{x - b}{d_0} \right\rbrace }\\
	&\leq d_0 \  \Prob\normalbrack{a \ Z_1 > \dfrac{x - b}{d_0}}\\
	&= d_0 \ \Prob\normalbrack{a \ d_0 \ Z_1 + b> x}.
	\end{align*}
\end{proof}
\begin{proof}[Proof of \autoref{cor:lighter_tail_prior}]
	Since $X_i$ is not affinely lighter tailed than $Z_1$ it can be shown that for any $\theta \in \Theta $ there exists an affine function $a_\theta:\R\to\R$ such that $a_\theta(X_i)$ is not affinely lighter tailed than $w_{\theta, i}(Z_1)$, where $w_{\theta, i}:\R \to \R$ is the function from \autoref{thm:tail_rate_bound}, implying 
	\[
	\Fbar_{\abs{w_{\theta, i}(Z_1)}} = o\normalbrack{\Fbar_{\abs{a_\theta(X_i)}}}.
	\]
	Then by \autoref{thm:tail_rate_bound} we obtain that for any $\theta \in \Theta$
	\[
	\Fbar_{\abs{g_{\theta, i}(Z)}} = o\normalbrack{\Fbar_{\abs{a_\theta(X_i)}}},
	\]
	which implies that the set of optimal generative networks $\mathcal{G}^*(Z, X)$ is empty. 
\end{proof}
\section{Basic Definitions}
\label{appendix_definitions}
\begin{definition}[Activation Function]
	\label{def:activation_function}
	A function $\phi:\R\to\R$ that is Lipschitz continuous, monotonic and satisfies $\phi(0)=0$ is called {\rm activation function}.
\end{definition}
\begin{remark}
	\autoref{def:activation_function} comprises a large class of functions found in literature \cite{maxout, prelu_he, relu_nair, efficient_backprop}.
\end{remark}
\begin{definition}[Deep Dense Sigmoidal Flow]
	\label{def:ddsf}
	Let $L, d_{0}, \dots, d_L \in \mathbb{N}$ such that $d_0 = d_L = 1$. Moreover, for $l = 1, \dots, L$ let
	$a^{(l)} \in \R^{d_{l}}_+, \  b^{(l)} \in \R^{d_{l}}$ and
	$w^{(l)} \in \R^{d_{l} \times d_{l}}, \ u^{(l)} \in \R^{d_{l} \times d_{l-1}}$ two non-negative matrices for which their
	row-wise sum is equal to $1$.
	Furthermore, let $\Psi: \R \to [0,1]$ define an invertible CDF.
	A function 
	\begin{align*}
		f:\R \times \Theta &\to \R\\
		(h^{(0)}, \theta) &\mapsto h^{(L)}
	\end{align*}
	where $h^{(L)}$ is defined recursively for $l = 1, \dots, L$ through
	\begin{equation}
	h^{(l)} = \Psi^{-1}\normalbrack{
		w^{(l)} \ 
		\Psi\normalbrack{
			{a}^{(l)} \odot u^{(l)} \  h^{(l-1)}
			+  b^{(l)}
		}
	}
	\label{eq:ddsf_map}
	\end{equation}
	is called a \textit{deep dense sigmoidal flow}.
\end{definition}
\newpage
\section{Additional Numerical Results}
\label{appendix_numerical_results}
\begin{figure}[htp]
	\centering
	\begin{subfigure}{.5\textwidth}
		\centering
		\includegraphics[width=\textwidth]{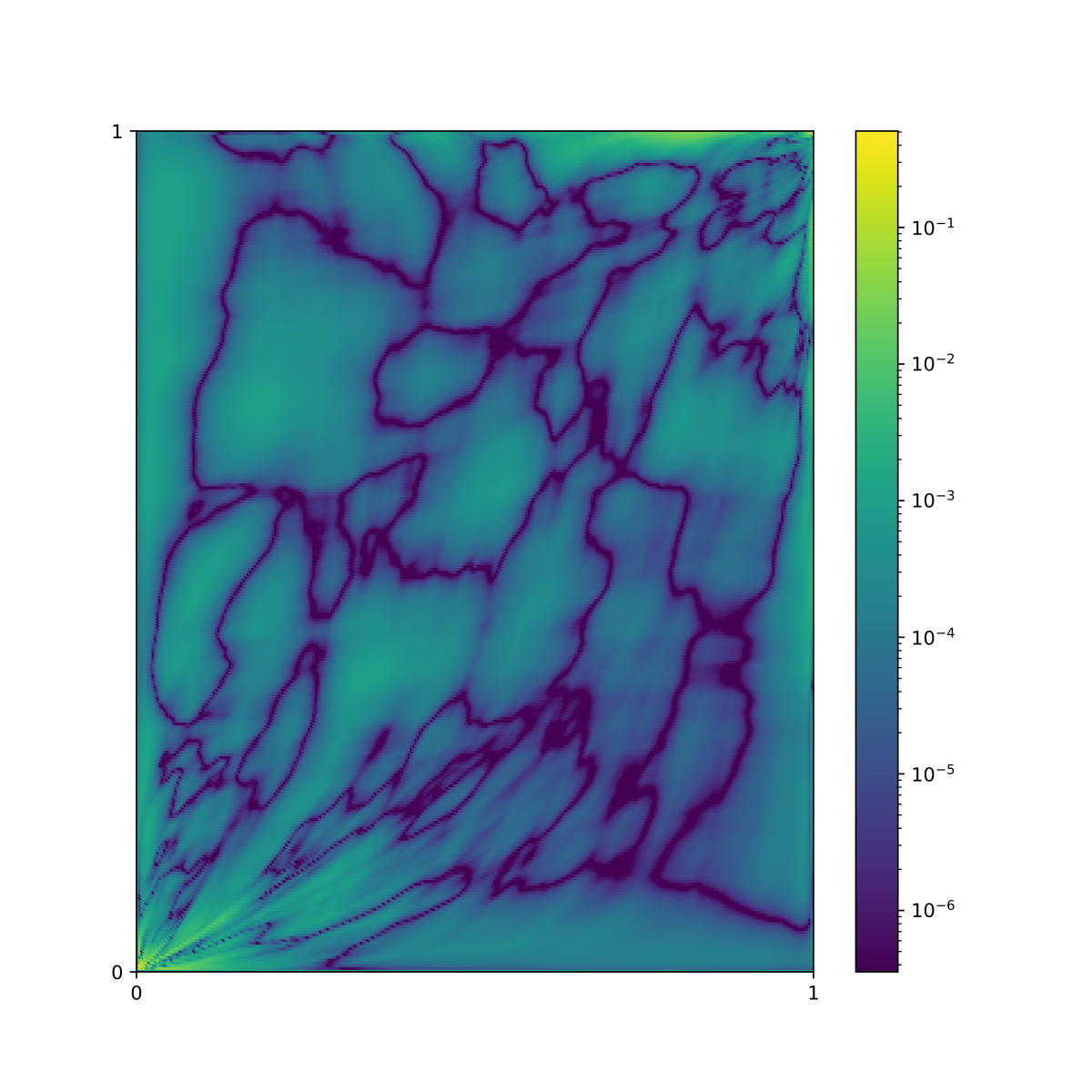}
		\caption{Clayton}
		\label{fig:s}
	\end{subfigure}
	\begin{subfigure}{.5\textwidth}
		\centering
		\includegraphics[width=\textwidth]{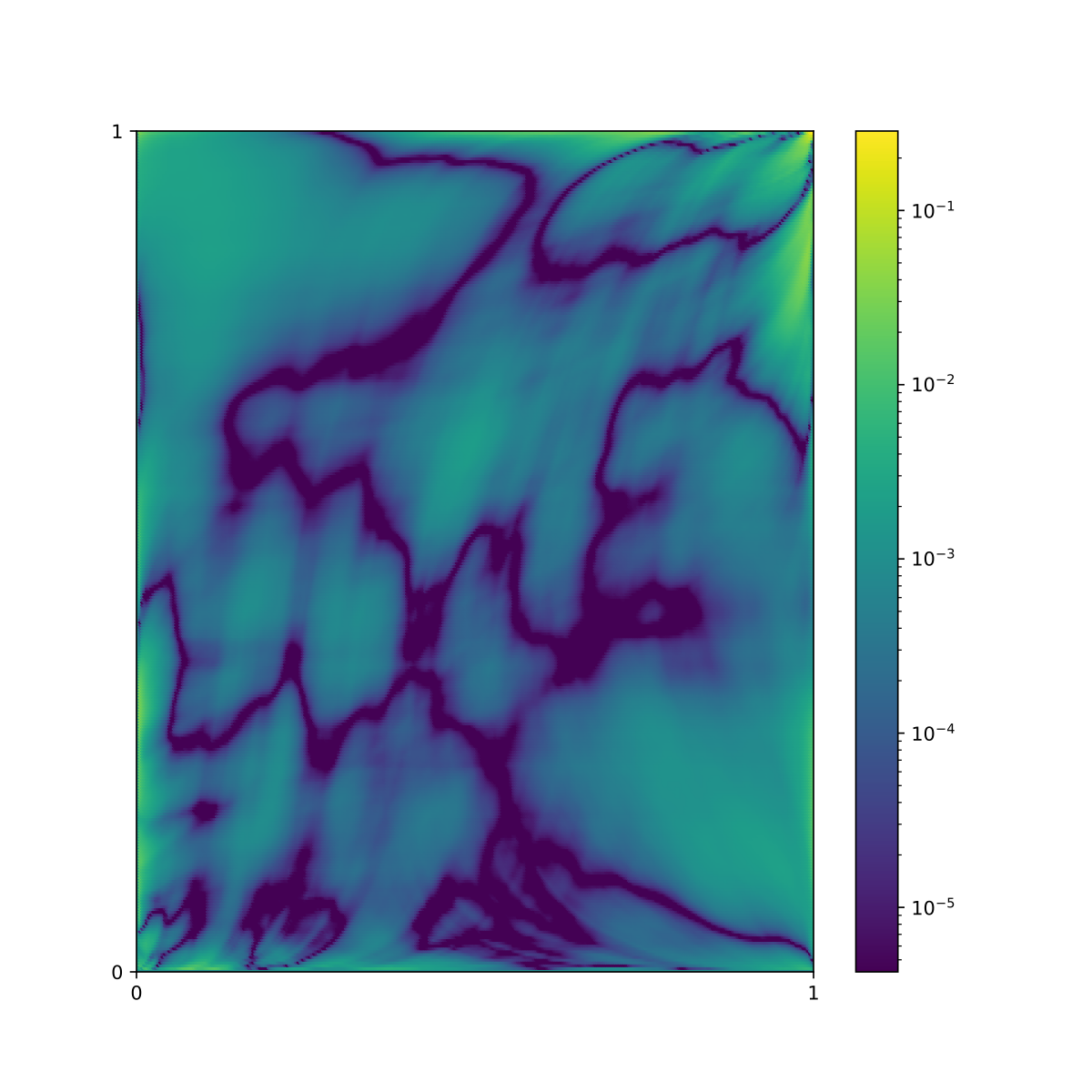}
		\caption{Frank}
		\label{fig:}
	\end{subfigure}
	\begin{subfigure}{.5\textwidth}
		\centering
		\includegraphics[width=\textwidth]{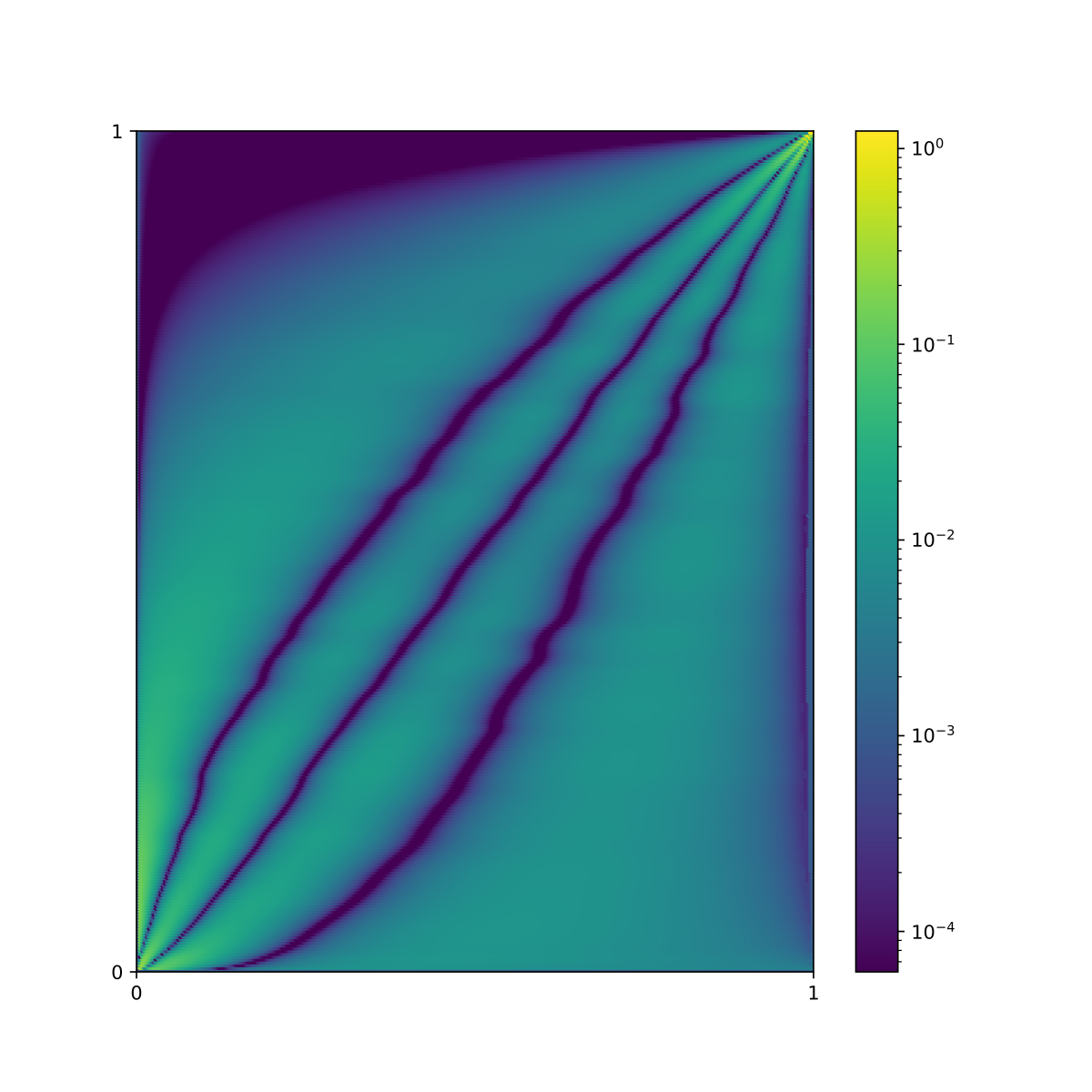}
		\caption{Gumbel}
		\label{fig:b}
	\end{subfigure}
	\caption{Pointwise evaluation of the JSD of the Clayton, Frank and Gumbel theoretical and copula flow densities.}
	\label{fig:jsd_pointwise}
\end{figure}

\end{document}